\documentclass{article}
\PassOptionsToPackage{dvipsnames}{xcolor}
\usepackage{iclr2027_conference,times}

\usepackage{microtype}
\usepackage{amsmath, amssymb, amsthm, mathtools, bm}
\usepackage{booktabs, array, multirow}
\usepackage{graphicx}
\usepackage{float}
\usepackage{xcolor}
\usepackage{enumitem}
\usepackage{xspace}
\usepackage{tikz}
\usepackage{bbm}
\usetikzlibrary{arrows.meta, positioning, fit, backgrounds, calc}
\usepackage{pgfplots}
\pgfplotsset{compat=1.16}
\usepackage[colorlinks=true, linkcolor=NavyBlue, citecolor=NavyBlue,
            urlcolor=NavyBlue]{hyperref}

\iclrfinalcopy  

\renewcommand{\arraystretch}{1.12}

\newtheorem{proposition}{Proposition}

\newcommand{\scope}{\textsc{Scope}\xspace}
\newcommand{\scopefullname}{Score-Isolated Control Optimization with Provenance-Bound Evidence}

\csname @@input\endcsname physiq_s115_lingbot_claims_generated.tex

\title{SCOPE: Score-Isolated Agentic Optimization\\
for Video World Models}
\author{Yuhua Jiang\\
Tsinghua University
\And Jiaming Wang\\
National University of Singapore
\AND Qingbin Liu\\
Tencent
\And Feifei Gao\\
Tsinghua University}
\hypersetup{
  pdfauthor={Yuhua Jiang, Jiaming Wang, Qingbin Liu, Feifei Gao},
  pdftitle={SCOPE: Score-Isolated Agentic Optimization for Video World Models}
}

\begin{document}
\maketitle
\fancyhead{}
\renewcommand{\headrulewidth}{0pt}

\begin{abstract}
\noindent
Video world models are increasingly used as simulators for planning and embodied decision making, yet improving them at inference time introduces a subtle evaluation problem: prompts, samplers, verifiers, and selectors may evolve together, making it difficult to attribute gains or prevent held-out feedback from shaping the final policy. We introduce \scope (\emph{\scopefullname}), a framework for auditable inference-time adaptation of frozen video world models. \scope represents external controls as a typed state, updates this state only through bounded changes supported by development evidence, and freezes the resulting policy before held-out evaluation. On Physics-IQ benchmark, \scope improves over the exact frozen base by $+14.24$ (95\% CI $[+8.10,+21.23]$). Controlled ablations further identify gains from scene specification, sampling, and learned selection, while the margin over the strongest matched agentic baseline remains unresolved. Cross-backbone and prospective evaluations reveal a complementary result: useful inference-time updates exist, but their benefits do not transfer uniformly across models and settings. Together, these findings suggest that reliable inference-time adaptation requires not only better proposals, but also a principled mechanism for deciding which updates should become part of the deployed system.
Code is available at \url{https://github.com/YuhuaJiang2002/SCOPE}.

\end{abstract}

\section{Introduction}
\label{sec:intro}
Video world models provide a flexible interface for predicting future visual states and
are increasingly used for simulation, planning, and embodied decision making. Yet
visually plausible videos are not necessarily physically consistent, and retraining a
large backbone for every new task is expensive. This motivates inference-time
adaptation: improving a frozen model through external controls such as prompts,
sampling strategies, verifiers, retrieval, and candidate selection.

The challenge is that these controls are typically coupled. An agent may rewrite a
prompt, alter the sampler, call a tool, rerank candidates, and reuse previous experience
within the same workflow. If several choices change at once, a higher final score does
not identify which intervention was responsible. More importantly, if held-out outcomes
are reused to choose later controls, the system can adapt to the evaluation itself. We
refer to this as the \emph{inference-control evaluation gap}: the gap between observing
an improvement and establishing that it comes from a well-defined inference-time
intervention.

We address this problem with \scope, short for \emph{\scopefullname}. \scope treats the
external inference harness as an explicit state $\Omega_r$ containing typed text,
sampler, verifier, and reward controls. At each development round, an agent proposes a
bounded update. The update is accepted only using a predefined development objective;
otherwise the previous state is retained. The final route, including an exact Frozen
Base fallback, is fixed before held-out evaluation. This design separates three objects
that are often conflated in agentic systems: the proposed update, the deployed control
state, and the score used to evaluate it.

We evaluate \scope from controlled components to the full inference-time system.
SceneLang conditioning, a sampler modification, and supervised best-of-$N$ selection
each improve their matched controls. On the 40-scene Physics-IQ common-base protocol,
the complete procedure reaches a P-IQ score of $34.94$ and improves over exact Base by
$+14.24$ (CI $[+8.10,+21.23]$), although the margin over the strongest matched
Qwen-Image-Agent-style baseline remains unresolved. Cross-backbone experiments on
Wan2.2 and CogVideoX, together with P-AI, OpenS2V-Eval, and prospective evaluations,
show a complementary result: improvements found during development do not reliably
transfer across all backbones, metrics, or unseen tasks. This distinction between
\emph{finding} useful updates and \emph{reliably selecting} them is central to our
analysis.

\paragraph{Contributions.}
Our contributions are threefold:
\begin{itemize}[leftmargin=*,nosep]
    \item We formulate the inference-control evaluation gap for agentic adaptation of
    frozen video world models and separate proposed updates, deployed controls, and
    held-out evaluation.
    \item We introduce \scope, a typed inference-control framework with bounded updates,
    an exact Base fallback, and a score-isolated update rule that prevents held-out
    scores from changing the deployed state.
    \item We provide matched experiments across controls, backbones, and benchmarks.
    \scope improves over exact Base on the matched Physics-IQ protocol, while ablations
    and prospective evaluations identify where component gains transfer and where
    selection remains unreliable.
\end{itemize}

\section{Related Work}
\label{sec:related}

\paragraph{Physics-aware video generation.}
Training-time methods inject physical supervision, imagined trajectories, latent motion
priors, or geometry rewards into model parameters
\citep{physisforcing2026,sift2026,lamo2026,geoalign2026}. Retrieval-based methods instead
augment generation with external physical knowledge~\citep{physrag2026}, while
PhyGround and WorldReasonBench stress-test physical reasoning and future-state
prediction~\citep{phyground2026,worldreasonbench2026}. Our setting instead keeps the video backbone fixed and studies reversible
inference-time controls under matched evaluation.

\paragraph{Inference-time alignment and scaling.}
\textsc{WMReward}, \textsc{VIGOR}, and \textsc{VHS} rank or steer generated samples with
latent-world, geometric, or hidden-state critics
\citep{wmreward2026,vigor2026,vhs2026}; iterative refinement and repeated sampling spend
additional test-time compute through different operators
\citep{ipr2026,brown2024large,ma2025inference}. These methods primarily optimize a particular search or reward mechanism.
\scope instead places prompts, samplers, verifiers, and rewards in a common typed
control space and evaluates each change against a matched protocol-local Base.

\paragraph{Agentic generation and tool orchestration.}
A closer comparison comes from agentic systems that improve generation through feedback,
planning, and reusable tools.
Qwen-Image-Agent closes context gaps through tool-augmented feedback, NEWTON plans
physically grounded videos, and LingBot-World structures interactive world
generation~\citep{qwenimageagent2026,newton2026,lingbotworld2026}. GenEvolve and
COMFYCLAW further organize tool calls and reusable workflow experience for visual
generation~\citep{genevolve2026,comfyclaw2026}. Tool-augmented systems can change retrieval, segmentation, geometry, and proposal
quality simultaneously, making attribution difficult. \scope therefore contributes a
score-isolated update rule rather than a new search, segmentation, or foundation model.
The selected route or Base fallback is fixed before held-out scoring, which enables
matched ablations of individual controls.

\paragraph{Persistent skills and self-evolving harnesses.}
SkillOpt, SkillCAT, and SkillDAG update or select reusable agent skills
\citep{skillopt2026,skillcat2026,skilldag2026}; Library Drift and Harness Updating Is Not
Harness Benefit expose failure modes in evolving skill libraries
\citep{librarydrift2026,harnessbenefit2026}. \scope similarly maintains a persistent typed control state, but restricts updates to
development evidence while keeping validation and confirmation read-only. This
separates the mechanics of updating an inference harness from evidence that the update
improves held-out performance.

\paragraph{Adaptive evaluation and safe updates.}
Reusable-holdout methods study validity under repeated adaptive queries, while
high-confidence policy evaluation asks whether a proposed policy is safe to deploy from
a conservative lower bound~\citep{dwork2015reusable,thomas2015highconfidence}.
These ideas motivate the score-isolation principle in \scope. The method prevents
direct adaptation to held-out scores, while repeated development on the same ledger can
still overfit; prospective evaluation on disjoint tasks therefore remains necessary.


\begin{figure}[t]
\centering
\resizebox{\linewidth}{!}{%
\begin{tikzpicture}[
  font=\scriptsize,
  >=Latex,
  panel/.style={draw=black!18, fill=black!1, rounded corners=4pt, line width=0.45pt},
  corepanel/.style={draw=RoyalPurple!70!black, fill=RoyalPurple!2, rounded corners=4pt, line width=0.9pt},
  evolvepanel/.style={draw=ForestGreen!72!black, fill=ForestGreen!2, rounded corners=4pt, line width=0.9pt},
  box/.style={draw=black!62, fill=white, rounded corners=2.5pt, line width=0.55pt,
              align=center, inner sep=3.5pt, minimum height=0.86cm, text width=2.0cm},
  context/.style={box, draw=black!42, fill=black!2, text width=1.95cm},
  tool/.style={box, draw=black!45, fill=black!2, text width=2.05cm},
  statecore/.style={box, draw=RoyalPurple!78!black, fill=RoyalPurple!9, line width=1.0pt, text width=2.2cm},
  editcore/.style={box, draw=RoyalPurple!65!black, fill=RoyalPurple!6, line width=0.8pt, text width=2.1cm},
  policycore/.style={box, draw=NavyBlue!76!black, fill=NavyBlue!7, line width=0.9pt, text width=2.2cm},
  freezecore/.style={box, draw=NavyBlue!82!black, fill=NavyBlue!10, line width=1.05pt, text width=2.05cm},
  genbox/.style={box, draw=black!75, fill=black!3, text width=2.1cm},
  acceptgate/.style={box, draw=ForestGreen!78!black, fill=ForestGreen!8, line width=1.0pt, text width=2.35cm},
  accept/.style={box, draw=ForestGreen!82!black, fill=ForestGreen!10, line width=1.0pt, text width=2.0cm},
  reject/.style={box, draw=Red!72!black, fill=Red!6, line width=0.9pt, text width=1.85cm},
  badge/.style={circle, draw=white, fill=RoyalPurple!82!black, text=white, font=\tiny\bfseries,
                inner sep=0pt, minimum size=0.42cm},
  badgeblue/.style={badge, fill=NavyBlue!82!black},
  badgegreen/.style={badge, fill=ForestGreen!82!black},
  badgered/.style={badge, fill=Red!78!black},
  invariant/.style={draw=NavyBlue!78!black, fill=NavyBlue!5, rounded corners=2.5pt,
                    line width=0.85pt, align=center, inner sep=3pt, text width=10.8cm,
                    font=\tiny},
  arr/.style={->, draw=black!70, line width=0.6pt, rounded corners=2pt},
  softarr/.style={->, draw=black!45, line width=0.5pt, dashed, rounded corners=3pt},
  looparr/.style={->, draw=ForestGreen!78!black, line width=1.05pt, rounded corners=3pt},
  lbl/.style={font=\tiny, inner sep=1.5pt, fill=white, align=center},
  looplbl/.style={font=\tiny\bfseries, inner sep=2pt, fill=ForestGreen!7,
                  text=ForestGreen!55!black, align=center},
  heading/.style={font=\tiny\bfseries, text=black!58, anchor=west},
  coreheading/.style={font=\tiny\bfseries, text=RoyalPurple!65!black, anchor=west,
                     fill=RoyalPurple!2, inner xsep=2pt, inner ysep=0.5pt},
  evolveheading/.style={font=\tiny\bfseries, text=ForestGreen!55!black, anchor=west,
                       fill=ForestGreen!2, inner xsep=2pt, inner ysep=0.5pt}
]
\node[panel, minimum width=13.15cm, minimum height=1.34cm] at (5.78,0.15) {};
\node[corepanel, minimum width=13.15cm, minimum height=1.48cm] at (5.78,-1.87) {};
\node[evolvepanel, minimum width=13.15cm, minimum height=1.72cm] at (5.78,-4.02) {};
\node[heading] at (-0.55,0.95) {proposal context (replaceable infrastructure)};

\node[context] (input) at (0.18,0.15) {prompt\\first frame};
\node[tool] (qwen) at (2.42,0.15) {VLM feedback\\or Director card};
\node[tool] (tools) at (4.82,0.15) {tool calls\\\texttt{image\_search} /\\\texttt{sam\_segment}};
\node[tool] (ledger) at (7.20,0.15) {frozen ledger\\references, masks,\\event notes};
\node[context] (actions) at (9.72,0.15) {proposal set\\$\mathcal A_s$};
\node[context, text width=1.45cm] (base) at (12.05,0.15) {Base\\fallback};

\node[statecore] (omega) at (0.70,-1.87) {\textbf{typed control state}\\$\Omega_r$: text $\cdot$ noise\\verifier $\cdot$ reward};
\node[editcore] (cards) at (3.42,-1.87) {bounded edits\\add / delete / replace\\retirable cards};
\node[editcore] (guard) at (6.18,-1.87) {provenance + guards\\conflicts, risk,\\Base fallback};
\node[policycore] (inner) at (8.98,-1.87) {score-blind policy\\$q_{\Omega_r}$, captions,\\reward, task evidence};
\node[freezecore] (select) at (11.65,-1.87) {\textbf{freeze before score}\\$a^\star(s)$ or Base};
\node[badge] at (-0.40,-1.31) {1};
\node[badgeblue] at (10.55,-1.31) {2};

\node[genbox] (wm) at (1.15,-4.02) {Frozen Wan2.2 WM\\$p_\theta(y\mid x_0,c)$\\weights fixed};
\node[genbox] (gen) at (4.00,-4.02) {generate mp4\\from frozen\\choice};
\node[acceptgate] (outer) at (6.95,-4.02) {\textbf{development acceptance rule}\\objective improves\\and all guards pass};
\node[accept] (accept) at (9.82,-4.02) {accept $\Omega_{r+1}$\\before held-out\\evaluation};
\node[reject] (reject) at (12.25,-4.02) {reject edit\\preserve $\Omega_r$};
\node[badgegreen] at (5.65,-3.40) {3};

\draw[arr] (input) -- (qwen);
\draw[arr] (qwen) -- (tools);
\draw[arr] (tools) -- (ledger);
\draw[arr] (ledger) -- (actions);
\draw[softarr] (base) -- (actions);
\draw[arr] (actions.south) -- ++(0,-0.55) -| (cards.north);
\draw[arr] (omega) -- (cards);
\draw[arr] (cards) -- (guard);
\draw[arr] (guard) -- (inner);
\draw[arr] (inner) -- (select);
\draw[arr] (select.south) -- ++(0,-0.44) -| (gen.north);
\draw[arr] (wm) -- (gen);
\draw[arr] (gen) -- (outer);
\draw[arr] (outer.east) -- (accept.west);
\draw[arr] (outer.south) -- ++(0,-0.24) -| (reject.south);
\node[lbl] at (11.15,-5.05) {otherwise};
\draw[looparr] (accept.south) -- ++(0,-0.82) -- (-0.95,-5.22) -- (-0.95,-1.87) -- (omega.west);
\node[looplbl] at (4.45,-5.22) {\textbf{3 score-isolated state update:} an accepted edit becomes the next typed state};
\draw[softarr] (reject.north) .. controls (12.35,-3.00) and (12.55,-2.55) .. (select.south east);
\node[invariant] at (6.05,-5.72) {\textbf{4 held-out scores are read-only:} validation / confirmation cannot change $\Omega_R$ or the frozen route $a^\star(s)$.};
\node[badgeblue] at (0.72,-5.72) {4};
\node[coreheading] at (-0.55,-1.03) {\scope CORE A: typed, score-blind control policy};
\node[evolveheading] at (-0.55,-3.08) {\scope CORE B: score-isolated state update};
\end{tikzpicture}%
}
\caption{\textbf{Overview of \scope.}
Proposal mechanisms in the gray band are replaceable. \scope maintains
\textbf{(1)} a typed state over text, noise, verifier, and reward controls;
\textbf{(2)} a score-blind deployment choice with an exact Base fallback;
\textbf{(3)} a development-time state update that accepts or rejects bounded edits; and
\textbf{(4)} a score-isolation constraint that keeps held-out outcomes read-only.
Only the external inference state changes; model weights remain frozen.}
\label{fig:agentic-context}
\label{fig:skillgate}
\end{figure}
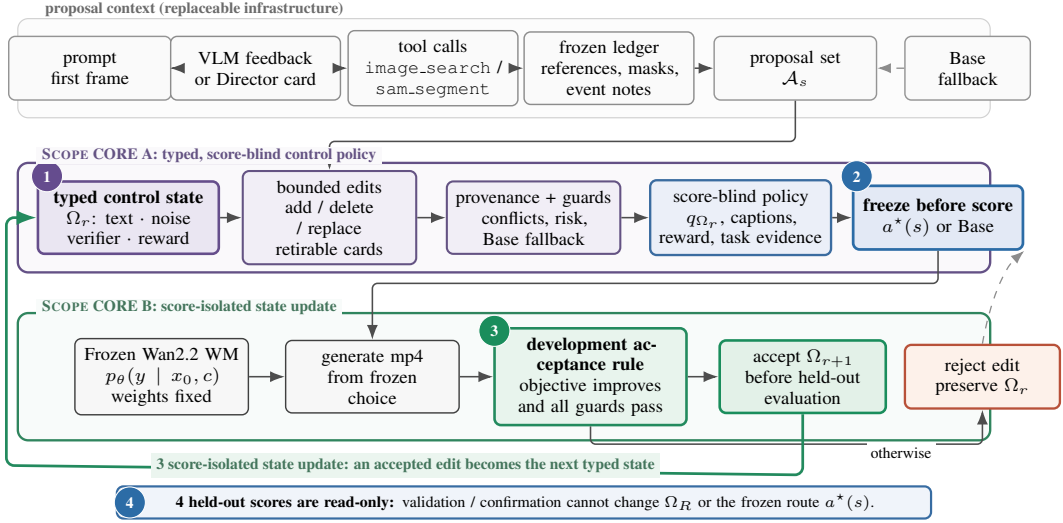

\section{Method}
\label{sec:skill}

\subsection{Problem Formulation and Overview}
\label{sec:problem-formulation}

Let $p_{\theta}(x_{1:T}\mid x_0,c)$ denote a pretrained video world model,
where $x_0$ is the initial observation, $c$ is the conditioning signal, and
$x_{1:T}$ is the generated trajectory. Throughout adaptation, the model
parameters $\theta$ remain frozen. Rather than modifying the world model
itself, we adapt an external inference-time control state
\begin{equation}
\Omega \in \mathcal{C},
\label{eq:control-space}
\end{equation}
where $\mathcal{C}$ denotes the space of admissible inference controls.
A control state may contain textual directives, sampling configurations,
verification or abstention rules, and reward-based selectors.

Starting from an initial state $\Omega_0$, an agent proposes a sequence of
bounded edits $\delta_0,\ldots,\delta_{R-1}$. State evolution is governed by
a development-only update operator,
\begin{equation}
\Omega_{r+1}
=
\mathcal{U}_r(\Omega_r,\delta_r;D_r),
\qquad
r=0,\ldots,R-1,
\label{eq:control-state-evolution}
\end{equation}
where $D_r$ is the sealed development ledger available at round $r$.
Each edit modifies a declared component of the current control state, while
$\mathcal{U}_r$ determines whether the edit is committed. Crucially,
$\mathcal{U}_r$ has access only to admissible development evidence; held-out
evaluation outcomes are never valid inputs to the update process.

After $R$ development rounds, the resulting state $\Omega_R$ is frozen.
For an evaluation instance $s$, the frozen state induces the deployment
policy
\begin{equation}
\pi_{\Omega_R}(s)
=
F(\Omega_R,x_s,\mathcal{A}_s),
\label{eq:deployment-policy-overview}
\end{equation}
where $x_s$ denotes the available input and $\mathcal{A}_s$ is the
corresponding candidate set. The complete routing rule, including the exact
Frozen Base fallback, is fixed before the held-out evaluation score is
revealed.

This formulation separates two questions that are often conflated in
inference-time adaptation: \emph{which intervention improves development
behavior}, and \emph{which intervention is admissible for deployment}.
We therefore formulate adaptation as constrained control-state optimization
around a frozen world model, rather than unconstrained search over prompts,
samplers, or candidate outputs.

\subsection{SCOPE: Structured Inference-Time Control}
\label{sec:scope-control}

\scope instantiates the above formulation with an explicit and persistent
control state. At round $r$, we decompose
\begin{equation}
\Omega_r
=
\left(
\Omega_r^{\mathrm{text}},
\Omega_r^{\mathrm{sample}},
\Omega_r^{\mathrm{verify}},
\Omega_r^{\mathrm{select}}
\right),
\label{eq:scope-state}
\end{equation}
where the four components respectively encode textual directives, sampling
controls, verification and abstention rules, and reward-based selection
mechanisms.

The decomposition makes interventions explicit and independently editable.
More importantly, it separates four conceptually distinct operations:
\emph{proposal}, \emph{commitment}, \emph{deployment}, and
\emph{evaluation}. At each round, a proposer suggests a bounded modification
to one control axis; the update rule determines whether that modification
enters the persistent state; the final frozen state determines deployment;
and held-out evaluation measures the resulting policy without modifying it.

Formally, an axis-specific proposer generates
\begin{equation}
\delta_r
=
P_r(\Omega_r,D_r),
\label{eq:proposal-rule}
\end{equation}
where $P_r$ may itself use VLM feedback, Director modules, retrieval,
external tools, or predefined proposal families. \scope does not constrain
the internal form of $P_r$; instead, it specifies the interface through which
a proposal is allowed to modify persistent inference-time behavior.

Some components of $\Omega_r$ may themselves be learned during development.
For example, the reward selector can be trained on disjoint Physics-IQ
scenes. Once development terminates, however, its parameters and inputs are
treated as part of the frozen control state. At deployment it consumes only
label-free features and cannot be updated using held-out outcomes.

Figure~\ref{fig:agentic-context} summarizes this separation between proposal
generation, persistent control-state adaptation, frozen deployment, and
held-out evaluation.

\subsection{Development-Only Control-State Updates}
\label{sec:control-update}

We restrict each proposal $\delta_r$ to a bounded elementary
\emph{add}, \emph{delete}, or \emph{replace} operation on one declared
control axis. This restriction makes individual state changes attributable
and prevents a single proposal from silently modifying multiple components.

Let $J_{\mathrm{dev}}(\Omega;D_r)$ denote the predeclared development
objective evaluated on ledger $D_r$. We define the development contrast of
proposal $\delta_r$ as
\begin{equation}
\Delta_r^{\mathrm{dev}}(D_r)
=
J_{\mathrm{dev}}
\bigl(
U_r(\Omega_r,\delta_r);D_r
\bigr)
-
J_{\mathrm{dev}}
\bigl(
\Omega_r;D_r
\bigr),
\label{eq:development-contrast}
\end{equation}
where $U_r(\Omega_r,\delta_r)$ denotes the state obtained by applying the
proposed edit before acceptance testing.

In addition to improving the development objective, a proposal must satisfy
a predefined collection of admissibility constraints. Let
\begin{equation}
G_{r,k}
\bigl(
U_r(\Omega_r,\delta_r);D_r
\bigr)
\leq 0,
\qquad
k=1,\ldots,K,
\label{eq:admissibility-constraints}
\end{equation}
represent the individual guards, including motion, sharpness, identity
preservation, and static-video checks. We summarize their conjunction by
\begin{equation}
G_r(D_r)
=
\prod_{k=1}^{K}
\mathbbm{1}
\left[
G_{r,k}
\bigl(
U_r(\Omega_r,\delta_r);D_r
\bigr)
\leq 0
\right].
\label{eq:aggregate-guard}
\end{equation}

The persistent state evolves through the commit-or-retain rule
\begin{equation}
\Omega_{r+1}
=
\begin{cases}
U_r(\Omega_r,\delta_r),
&
\Delta_r^{\mathrm{dev}}(D_r)>0
\ \land\
G_r(D_r)=1,
\\[1mm]
\Omega_r,
&
\text{otherwise}.
\end{cases}
\label{eq:scope-commit}
\end{equation}

Equation~\ref{eq:scope-commit} gives every persistent intervention an explicit
acceptance criterion: an edit is committed only when it improves its
predeclared development contrast and satisfies every guard. Otherwise, the
incumbent is retained exactly.

Importantly, this rule does not assume that development improvement transfers
monotonically to unseen tasks, environments, or model backbones. Its purpose
is narrower: to specify which evidence is permitted to modify the deployed
system and to prevent held-out evaluation feedback from becoming an implicit
optimization signal.

\subsection{Frozen Deployment and Score Isolation}
\label{sec:score-isolation}

After the final development round, adaptation terminates and the state
$\Omega_R$ is frozen. For each evaluation instance $s$, the inference
procedure constructs an admissible candidate set
\begin{equation}
\mathcal{A}_s
=
\left\{
a_s^{(1)},\ldots,a_s^{(M)},a_s^{\mathrm{Base}}
\right\},
\label{eq:candidate-set}
\end{equation}
where $a_s^{\mathrm{Base}}$ is the exact output of the Frozen Base system and
is always retained as a fallback.

The deployed route is then fixed as
\begin{equation}
a_s^\star
=
\pi_{\Omega_R}(s)
=
F(\Omega_R,x_s,\mathcal{A}_s),
\label{eq:frozen-route}
\end{equation}
before the official evaluation score $y_s$ is revealed. Equivalently, the
complete deployment mapping
\begin{equation}
\Phi_{\Omega_R}:
(s,\mathcal{A}_s)
\longmapsto
a_s^\star
\label{eq:frozen-deployment-map}
\end{equation}
is precommitted before held-out evaluation.

Optional MAGE-lite memory may supply capability, task, experience, or
environment records to $F$. Such records are treated as part of the frozen
inference state: they may be constructed or updated during development, but
cannot be modified using held-out outcomes.

This ordering yields a simple score-isolation property.

\begin{proposition}[Score-isolation invariant]
\label{prop:score-isolation}
Fix the proposal randomness, sealed development ledgers, candidate
construction procedure, world model, evaluation inputs, and fallback rule.
Replacing the complete validation, confirmation, held-out, or official-score
ledger by arbitrary alternative values leaves both the final control state
$\Omega_R$ and every frozen deployment route $a_s^\star$ unchanged prior to
score release.
\end{proposition}

\begin{proof}
By Equation~\ref{eq:proposal-rule}, proposal $\delta_r$ is determined only by
the incumbent $\Omega_r$, sealed development ledger $D_r$, and fixed proposal
randomness. Equation~\ref{eq:scope-commit} likewise depends only on
development quantities associated with the current round. Therefore, by
induction over $r$, the complete state trajectory
$\Omega_0,\ldots,\Omega_R$ is invariant to any modification of held-out
scores. Equation~\ref{eq:frozen-route} subsequently depends only on the
frozen state $\Omega_R$, evaluation input $x_s$, candidate set
$\mathcal{A}_s$, and the fixed fallback rule. Hence every route $a_s^\star$
inherits the same invariance.
\end{proof}

The proposition establishes a deliberately narrow guarantee: held-out scores
cannot influence either control-state adaptation or deployment routing.
It does not imply that the proposal mechanism is optimal, that development
selection is unbiased, that development gains generalize to unseen settings,
or that \scope necessarily outperforms a matched alternative. These are
empirical questions addressed separately in Section~\ref{sec:experiments}.

\subsection{Provenance-Bound Update Records}
\label{sec:provenance}

Score isolation specifies which information may influence adaptation;
provenance binding makes the resulting state evolution auditable. We bind
each attempted transition to the exact incumbent state, proposal, and
development evidence from which it was produced.

Let $h(\cdot)$ denote the SHA-256 hash of a canonical serialization of its argument, let $b_r\in\{0,1\}$ denote the
commit decision, and let $\mathcal{R}_r$ contain the frozen outcomes and
reasons of all admissibility guards. Each attempted transition emits the
record
\begin{equation}
\mathcal{T}_r
=
\Bigl(
h(\Omega_r),
h(\delta_r),
h(D_r),
b_r,
\mathcal{R}_r,
h(\Omega_{r+1})
\Bigr).
\label{eq:transition-record}
\end{equation}

A proposal is admissible only if its recorded parent agrees with the current
incumbent,
\begin{equation}
h(\Omega_r^{\mathrm{parent}})
=
h(\Omega_r).
\label{eq:parent-hash}
\end{equation}
A committed proposal modifies exactly one declared control axis and advances
the corresponding revision. A rejected or retired proposal instead restores
the incumbent exactly, which implies
\begin{equation}
b_r=0
\quad\Longrightarrow\quad
\Omega_{r+1}=\Omega_r
\quad\Longrightarrow\quad
h(\Omega_{r+1})=h(\Omega_r).
\label{eq:retirement-invariant}
\end{equation}

Consequently, $\mathcal{T}_r$ distinguishes an edit that was merely proposed
from one that actually modified the persistent state, while binding every
committed change to its parent state, development evidence, and gate
decision.

\paragraph{Evidence roles.}
Only evidence explicitly designated for development and sealed before
evaluation is admissible to Equations~\ref{eq:proposal-rule}--
\ref{eq:scope-commit}. Validation, confirmation, official-score, held-out,
and test fields are excluded from the executable update interface. Each
admissible transition records the development-ledger hash, guard outcomes,
commit decision, and resulting state hash.

This distinction has a direct operational consequence. Altering any held-out
score leaves the admissible inputs to Equation~\ref{eq:scope-commit}
unchanged and therefore cannot alter the deployed state. Altering development
evidence, in contrast, changes $h(D_r)$ and defines a distinct
provenance-bound transition that must be evaluated independently. The
resulting procedure therefore provides both \emph{score isolation}---held-out
outcomes cannot drive adaptation---and \emph{transition traceability}---every
persistent state change can be replayed from its recorded parent, proposal,
evidence, and decision.

\section{Experiments}
\label{sec:experiments}

\subsection{Experimental Setup}
We evaluate \scope with two frozen video backbones, Wan2.2 and CogVideoX, across
Physics-IQ, P-AI (PAI-Bench-G), and OpenS2V-Eval. The main Physics-IQ comparison uses
40 scenes and reports scene-level composite scores. P-AI evaluates the matched methods on the robot split6 tasks with the official Qwen2.5-VL-72B judge and an
eight-dimensional quality score. OpenS2V-Eval contains $180$ items per backbone and is
used as a broader open-domain test. We additionally use fresh Physics-IQ and PhyGround
tasks for prospective evaluation.

All comparisons are performed within the same backbone and evaluation protocol.
Unless noted otherwise, deltas are paired against the corresponding Frozen Base and
uncertainty is estimated using a paired bootstrap over the protocol's independent unit.
The generator, candidate pool, deployment route, and analysis unit are fixed before
held-out outcomes are read. Secondary and post-hoc analyses are reported descriptively
rather than pooled with the primary comparison. Full protocol and statistical details
are provided in Appendix Table~\ref{tab:governance}.

\subsection{Main Results on Physics-IQ}
\label{sec:physiq-common-base}
Our primary experiment asks two questions under the same 40-scene protocol:
whether \scope improves over the exact Frozen Base and whether it improves over the
strongest matched alternative.

\ScopeCommonBaseResultParagraph

The result is clear relative to Frozen Base but less conclusive relative to the
strongest matched alternative. The direct \scope--Qwen comparison changes only
$8/40$ scenes ($3/32/5$ wins/ties/harms), which explains the wider uncertainty in that
contrast. Table~\ref{tab:headline} therefore reports the scene-equal estimand. Because
the official Physics-IQ composite is nonlinear, absolute rankings across alternative
aggregation schemes are not inferred from scene composites alone.

\csname @@input\endcsname scope_common_base_first_four_rounds_generated.tex

\begin{table}[t]
\centering
\scriptsize
\caption{\textbf{Common-base Physics-IQ comparison across two frozen video backbones.}
\ScopeCommonBaseTableCaptionClaim}
\label{tab:headline}
\label{tab:physiq-common-base}
\setlength{\tabcolsep}{1.6pt}
\resizebox{\linewidth}{!}{%
\begin{tabular}{@{}lcccc@{}}
\toprule
\textbf{Method} & \textbf{Wan P-IQ} & \textbf{Wan $\Delta$ (95\% CI)} & \textbf{CogVideoX P-IQ} & \textbf{CogVideoX $\Delta$ (95\% CI)}\\
\midrule
\csname @@input\endcsname physiq_common_base_crossbackbone_rows.tex
\bottomrule
\end{tabular}
}
\end{table}

\begin{figure}[H]
\centering
\includegraphics[width=\linewidth]{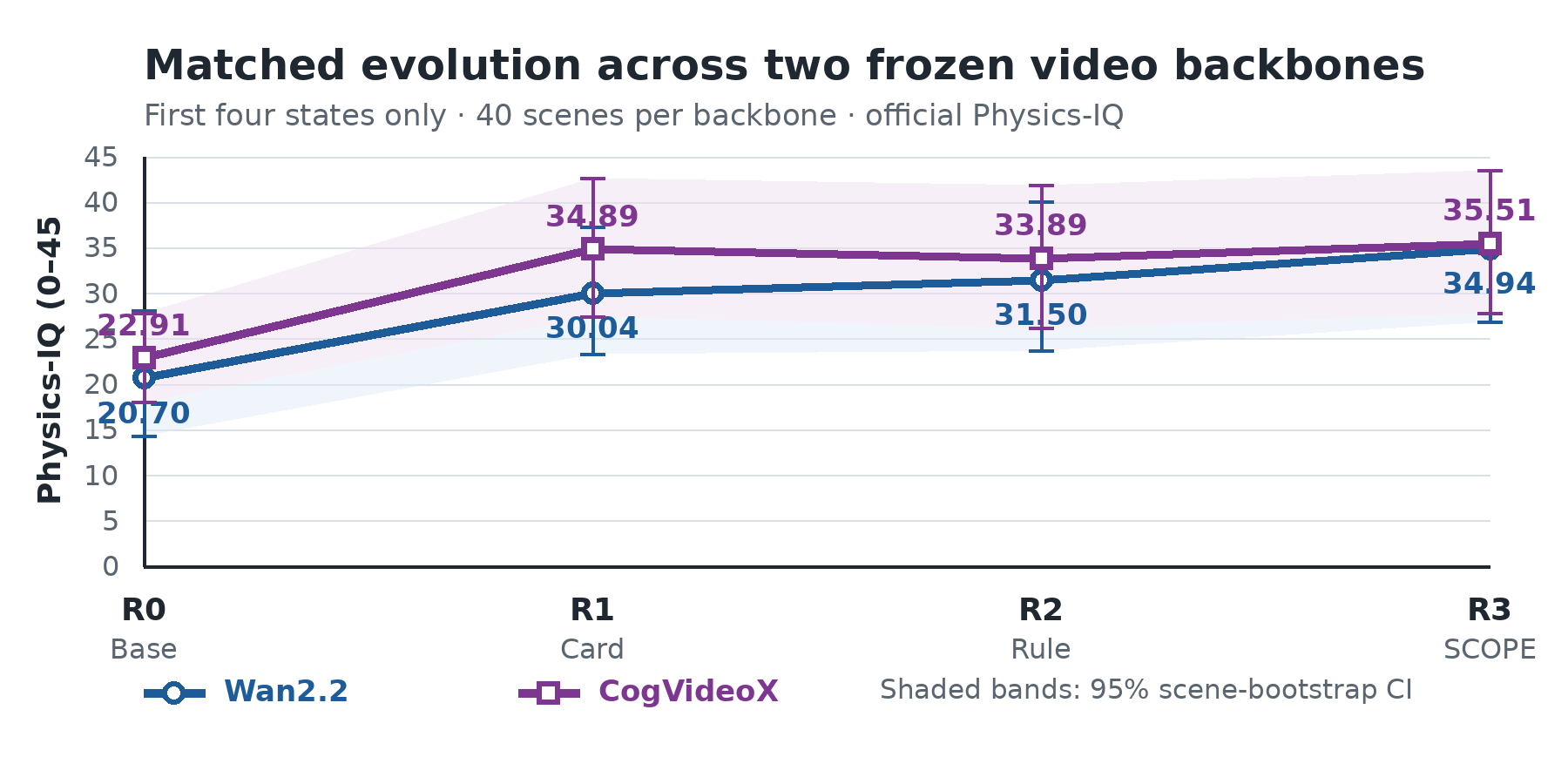}
\caption{\ScopeEvolutionCaption}
\label{fig:scope-common-base-evolution}
\end{figure}

\ScopeEvolutionParagraph



\begin{figure}[t]
\centering
\includegraphics[width=\linewidth]{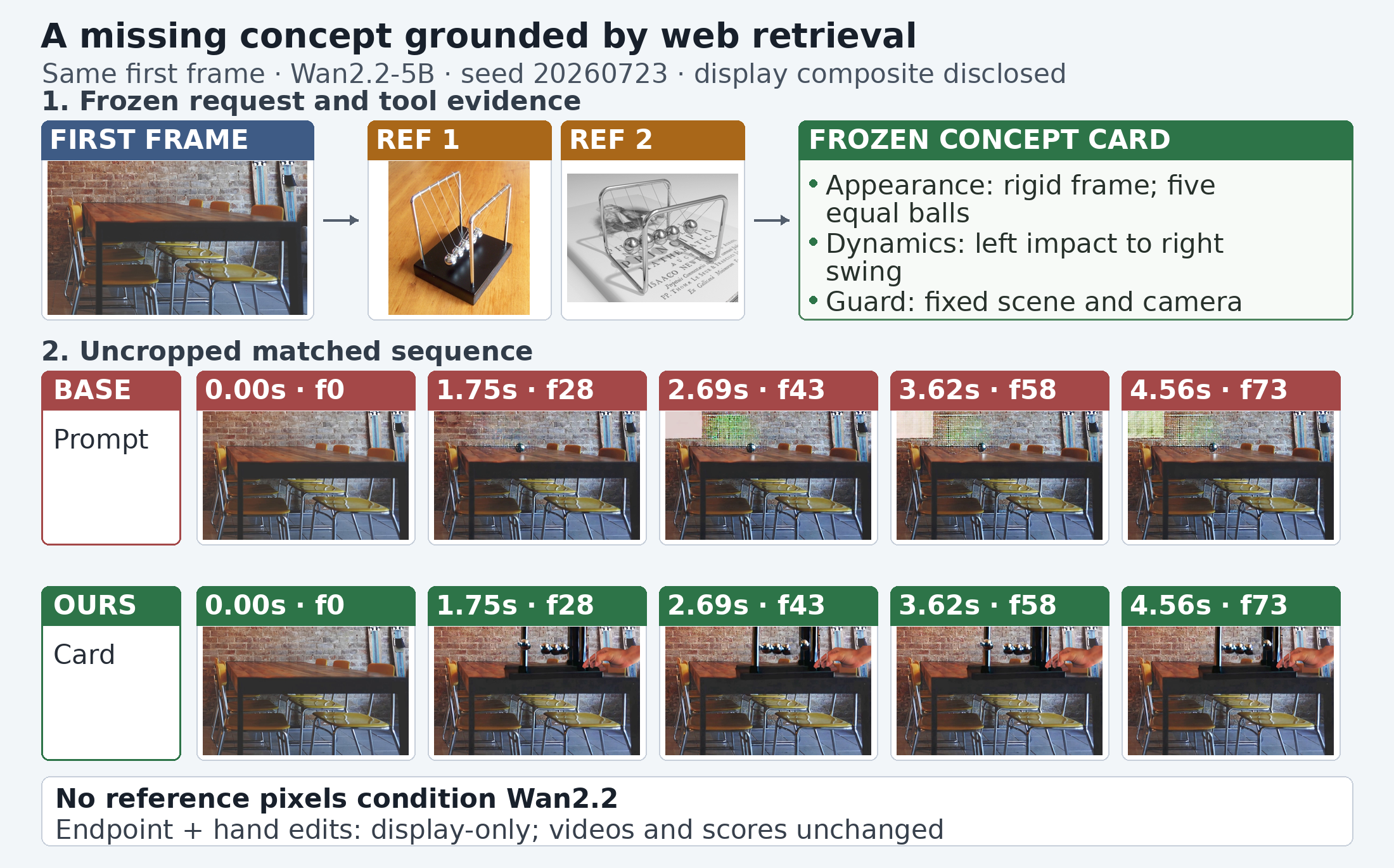}
\caption{\textbf{Example of retrieval-assisted generation on Newton's cradle.}
Retrieval provides a textual concept card but no reference pixels or visual
embeddings. Both Wan2.2 variants use the same temporal anchoring operation.
The visualization for OURS additionally includes the disclosed endpoint-motion
intervention and display-only hand replacements. This example is qualitative
and is not used as evidence for the aggregate tool-use effect.}
\label{fig:newton-web-case}
\end{figure}

\subsection{Cross-Backbone Generalization}
\label{sec:pai-crossbackbone}

We next study whether the observed behavior transfers when both the video
backbone and task distribution change. Table~\ref{tab:pai-headline} evaluates
all methods on the same robot split6 tasks using the official
Qwen2.5-VL-72B judge and the same eight-dimensional QS evaluation contract.

\begin{table}[t]
\centering
\scriptsize
\caption{\textbf{Cross-backbone results on P-AI (PAI-Bench-G).}
Deltas and paired 95\% CIs are relative to the Frozen Base of the same backbone.
SceneLang is the strongest fixed-policy comparator. The diagnostic \scope rows
report the better of two variants per backbone and therefore do not represent a
single transferable policy.}
\label{tab:pai-headline}
\setlength{\tabcolsep}{1.8pt}
\resizebox{\linewidth}{!}{%
\begin{tabular}{@{}lcccc@{}}
\toprule
\textbf{Matched method}
& \textbf{Wan Overall}
& \textbf{Wan $\Delta$ (95\% CI)}
& \textbf{CogVideoX Overall}
& \textbf{CogVideoX $\Delta$ (95\% CI)}\\
\midrule
\csname @@input\endcsname pai_crossbackbone_rows.tex
\bottomrule
\end{tabular}
}
\end{table}

Among the fixed policies, SceneLang achieves the highest point estimate on both
backbones ($77.02$ on Wan and $74.72$ on CogVideoX), although its Base-relative
confidence intervals include zero. The diagnostic \scope upper envelope gives
$77.47$ on Wan using the shared-global variant and $74.48$ on CogVideoX using
the no-global variant. Because these values are obtained by different variants,
they should not be interpreted as the performance of a single transferable
policy.

The shared-global variant further illustrates this backbone dependence. It
changes Overall performance by $+0.80$ (CI $[-1.38,+3.12]$) relative to Frozen
Base on Wan2.2, but by $-2.41$ (CI $[-5.04,+0.13]$) on CogVideoX. Although
neither interval is conclusive, the reversal in point estimates suggests that
the effectiveness of this component depends on the underlying backbone.
Complete Wan profiles are provided in Appendix Table~\ref{tab:pai-8d}.

\subsection{Open-Domain Generalization on OpenS2V-Eval}
\label{sec:opens2v-crossbackbone}

We next evaluate cross-backbone generalization in the broader open-domain setting
of OpenS2V-Eval. Within each backbone, all methods share the same frozen sampler
configuration, and routing decisions are fixed before evaluation.

\begin{table}[t]
\centering
\footnotesize
\caption{\textbf{Cross-backbone results on OpenS2V-Eval
($n{=}180$ per backbone).}
Scores are multiplied by $100$ and compared only within each backbone.
Mean$^{\dagger}$ and Weighted$^{\dagger}$ use the disclosed GPT-5.4 judge with
a Qwen fallback proxy. $^{\ddagger}$ denotes a prompt-only LingBot-style proxy.}
\label{tab:opens2v-headline}
\setlength{\tabcolsep}{2pt}
\resizebox{\linewidth}{!}{%
\begin{tabular}{@{}lcccc@{}}
\toprule
\textbf{Method}
& \textbf{Wan Mean$^{\dagger}$}
& \textbf{Wan Weighted$^{\dagger}$}
& \textbf{CogVideoX Mean$^{\dagger}$}
& \textbf{CogVideoX Weighted$^{\dagger}$}\\
\midrule
\csname @@input\endcsname opens2v_crossbackbone_rows.tex
\bottomrule
\end{tabular}
}
\end{table}
For a matched agentic baseline, we construct a LingBot-style prompt proxy that
injects a Qwen3.5-9B caption of the first frame into a fixed
Director$\rightarrow$Pilot template, while keeping the video backbone and
sampler unchanged. This baseline isolates the contribution of structured
agentic prompting within our evaluation protocol; it is not intended to
reproduce the trained LingBot-World system.

On Wan, \scope improves over this proxy by $+0.47$ in Mean
(CI $[-0.45,+1.43]$) and $+0.33$ in Weighted score
(CI $[-0.93,+1.62]$). On CogVideoX, the corresponding differences are
$+0.37$ (CI $[-0.48,+0.98]$) and $-0.40$
(CI $[-1.39,+0.51]$), respectively. All four confidence intervals include zero,
so these contrasts remain statistically unresolved. Notably, Frozen Base
achieves the highest Mean score on CogVideoX, whereas the LingBot-style proxy
slightly exceeds \scope on the Weighted metric. Taken together, these results
indicate that the gains from learned inference-time control are
backbone- and metric-dependent, rather than uniformly transferable across
configurations.

More broadly, the direction of the effect is not fully stable across
backbones. The Mean difference relative to the strongest comparator is slightly
positive on both Wan ($+0.47$) and CogVideoX ($+0.37$), whereas the Weighted
difference changes from $+0.33$ to $-0.40$. P-AI exhibits an even larger
directional change, from $+0.80$ to $-2.41$ relative to the same-backbone Base.
These results suggest a backbone--metric interaction rather than a
backbone-independent improvement.

Taken together, Physics-IQ, P-AI, and OpenS2V-Eval show that \scope is compatible
with multiple backbones and evaluation settings, but the effect of a particular control
is not invariant to backbone or metric. We therefore report cross-backbone results
separately rather than pooling scores across heterogeneous protocols.


\subsection{Prospective Evaluation on Unseen Tasks}
\label{sec:prospective-eval}

The previous experiments evaluate matched controls and cross-backbone behavior. We next
test whether development-time improvements persist when the selection policy is frozen
and applied to unseen tasks.

\begin{table}[t]
\centering
\small
\caption{\textbf{Prospective evaluations.}
Positive effects favor the proposed method; PhyGround reports proposal coverage
against its predeclared minimum.}
\label{tab:prospective-tests}
\setlength{\tabcolsep}{6pt}
\begin{tabular}{@{}llc@{}}
\toprule
\textbf{Evaluation} & \textbf{Reference / Criterion} & \textbf{Outcome} \\
\midrule
Fresh Physics-IQ
    & Base
    & $+0.096\;[-0.887,+1.077]$ \\
    & Frozen router
    & $-1.019\;[-2.159,+0.146]$ \\
\addlinespace[2pt]

Typed update
    & Random update
    & $-0.316\;[-5.984,+5.145]$ \\
\addlinespace[2pt]

Tool router
    & Always-on
    & $\mathbf{+0.660}\;[+0.016,+1.315]$ \\
    & No tool
    & $+0.083\;[-0.191,+0.381]$ \\
    & Random routing
    & $+0.002\;[-0.281,+0.299]$ \\
\addlinespace[2pt]

PhyGround
    & Min.\ proposal coverage: $8/32$
    & $4/32$ \\
\bottomrule
\end{tabular}
\end{table}

Table~\ref{tab:prospective-tests} summarizes four prospective tests.
On fresh Physics-IQ tasks, neither the comparison with Base nor that with the
frozen router establishes a clear improvement. The typed updater likewise does
not outperform a matched random update and degrades $15\%$ of scenes, exceeding
the predeclared $10\%$ tolerance. The grounded tool router improves over
always-on tool use, whereas its contrasts with no-tool and same-budget random
routing remain unresolved. PhyGround exposes a different failure mode: only
$4/32$ tasks produce a valid proposal, below the predeclared minimum coverage of
$8/32$. The proposal stage therefore fails its coverage gate before downstream
tool quality can be meaningfully assessed. Together, these results show that
development-time component gains do not yet yield robust improvement under
prospective evaluation.

\subsection{Ablation Studies}
\label{sec:main-ablation}

We analyze three complementary aspects of \scope: individual control axes,
control-state updates, and external tool use. Table~\ref{tab:ablation} summarizes
all matched ablations. Because these experiments use different controls and
evaluation units, their effects are interpreted independently rather than
combined additively.

\begin{table}[t]
\centering
\footnotesize
\caption{\textbf{Ablation studies of \scope.}
Each row reports a matched comparison and its 95\% CI. Component effects are
protocol-specific and should not be added across rows.}
\label{tab:ablation}
\setlength{\tabcolsep}{4.5pt}
\renewcommand{\arraystretch}{1.02}
\begin{tabular}{@{}lllc@{}}
\toprule
\textbf{Category} & \textbf{Variant} & \textbf{Comparison} &
\textbf{$\Delta$ (95\% CI)} \\
\midrule

\multirow{5}{*}{Control}
& Text (SceneLang)     & vs.\ Base
& $+5.62\;[+0.55,+10.55]$ \\
& Noise                & vs.\ fixed sampler
& $+1.26\;[+0.52,+2.01]$ \\
& Reward               & vs.\ matched random
& $+3.33\;[+1.49,+5.30]$ \\
& Reward               & vs.\ self-score
& $+0.66\;[-1.52,+2.92]$ \\
& Region-state prior   & identity metric
& $-7.28\;[-13.33,-2.53]$ \\

\midrule

\multirow{4}{*}{Update}
& Random update        & vs.\ $\Omega_0$
& $+4.53\;[+1.14,+9.17]$ \\
& Typed update         & vs.\ $\Omega_0$
& $+4.22\;[+0.81,+8.30]$ \\
& Typed update         & vs.\ random update
& $-0.32\;[-5.98,+5.15]$ \\
& No fallback          & vs.\ Base fallback
& $-0.773\;[-1.471,-0.203]$ \\

\midrule

\multirow{3}{*}{Tool}
& Retrieval            & vs.\ no tool
& $+0.562\;[-0.700,+1.967]$ \\
& Retrieval + SAM      & vs.\ no tool
& $+0.013\;[-2.936,+3.263]$ \\
& Incremental SAM      & vs.\ retrieval
& $-0.548\;[-3.390,+2.631]$ \\

\bottomrule
\end{tabular}
\end{table}

\paragraph{Component controls.}
The text and noise controls both improve their matched baselines, and the learned
reward selector outperforms matched random selection. Its advantage over self-score,
however, remains unresolved. The SceneLang interval is conditional on three fixed
seeds. The reward selector is trained with scene-grouped out-of-fold supervision and
uses only label-free features at inference. In contrast, the region-state prior
degrades the primary identity metric and is therefore excluded from the final control
state.

\paragraph{Control-state updates.}
Both the random and typed updates improve over the frozen incumbent $\Omega_0$,
indicating that updating the control state can be beneficial. The typed update,
however, does not outperform an equal-budget random update and degrades $15\%$ of
scenes, so the current experiment does not establish an advantage for typed transfer
itself. Removing the Base fallback produces a clear degradation, supporting fallback
as a conservative mechanism when an update is unreliable. The full $+14.24$
Physics-IQ improvement is an integrated system result and should not be interpreted
as the sum of these local ablations.

\paragraph{Tool use.}
Retrieval supplies frozen textual evidence rather than reference pixels or visual
embeddings, while SAM contributes first-frame geometric information without semantic
labels. None of the three tool-use contrasts is statistically resolved. Because the
evaluation contains multiple views nested within each scene, we use the scene-paired
analysis as the primary estimand. These results indicate that the overall Physics-IQ
gain cannot be attributed to tool use alone.

\section{Limitations}
\label{sec:limitations}

Our experiments expose three main limitations of the current approach.

First, strong candidate proposals do not necessarily imply reliable deployment
decisions. Several individual controls improve under matched evaluation, and the
complete \scope policy substantially improves over Frozen Base on Physics-IQ.
However, prospective experiments show that selecting when to apply these controls
remains substantially harder than generating useful candidates. In particular,
selection errors and insufficient proposal coverage limit performance on unseen tasks.

Second, the effectiveness of inference-time controls is not fully invariant across
backbones or metrics. Updates that improve one configuration can have smaller effects,
or even reverse direction, under another backbone. This limits the extent to which a
control learned in one setting can currently be treated as a transferable adaptation
rule.

Third, \scope relies on development-time evidence to decide which updates may enter
the deployed control state. Although held-out scores cannot directly modify the policy,
the quality of the resulting policy still depends on the diversity and representativeness
of the development tasks. More reliable adaptation will therefore require stronger
task-disjoint proposal generation, calibrated uncertainty estimates, and selectors
that generalize under distribution shift.

The Frozen Base fallback partially mitigates these limitations by allowing the system
to abstain from uncertain updates rather than forcing every proposal to be deployed.
Nevertheless, improving the reliability of this decision remains a central direction
for future work.

\section{Conclusion}
\label{sec:conclusion}

We introduce \scope, a framework for inference-time adaptation of frozen video world
models that explicitly separates proposal, update, deployment, and evaluation.
\scope represents prompts, sampler configurations, verifiers, and reward selectors
within a structured control state, updates this state using development evidence, and
freezes the resulting deployment policy before held-out evaluation.

On the Physics-IQ benchmark, \scope substantially improves over the exact
Frozen Base on both evaluated backbones. Controlled ablations further identify useful
contributions from scene specification, sampling, and learned selection, while the
advantage over the strongest matched agentic alternative remains unresolved.
Cross-backbone and prospective evaluations show that these gains are not uniformly
transferable, highlighting reliable selection under distribution shift as the main
remaining challenge.

More broadly, our results suggest that improving an agentic inference system requires
more than generating better candidate interventions. It also requires a principled
mechanism for deciding which interventions should become part of the deployed policy.
By making this distinction explicit, \scope provides a foundation for studying
inference-time adaptation as a controlled optimization problem rather than an
unstructured sequence of post-hoc improvements.

\clearpage
\bibliographystyle{iclr2027_conference}
\bibliography{references}

\begin{thebibliography}{27}
\providecommand{\natexlab}[1]{#1}
\providecommand{\url}[1]{\texttt{#1}}
\expandafter\ifx\csname urlstyle\endcsname\relax
  \providecommand{\doi}[1]{doi: #1}\else
  \providecommand{\doi}{doi: \begingroup \urlstyle{rm}\Url}\fi

\bibitem[Bai et~al.(2026)Bai, Wan, Zhou, Yu, You, and Tsang]{skilldag2026}
Tong Bai, Zhenglin Wan, Pengfei Zhou, Xingrui Yu, Yang You, and Ivor~W. Tsang.
\newblock {SkillDAG}: Self-evolving typed skill graphs for {LLM} skill selection at scale.
\newblock \emph{arXiv preprint arXiv:2606.03056}, 2026.

\bibitem[Brown et~al.(2024)Brown, Juravsky, Ehrlich, Clark, Le, R{\'e}, and Mirhoseini]{brown2024large}
Bradley Brown, Jordan Juravsky, Ryan Ehrlich, Ronald Clark, Quoc~V. Le, Christopher R{\'e}, and Azalia Mirhoseini.
\newblock Large language monkeys: Scaling inference compute with repeated sampling.
\newblock \emph{arXiv preprint arXiv:2407.21787}, 2024.

\bibitem[Bucciarelli et~al.(2026)Bucciarelli, Turri, Baraldi, Cornia, and Cucchiara]{vhs2026}
Davide Bucciarelli, Evelyn Turri, Lorenzo Baraldi, Marcella Cornia, and Rita Cucchiara.
\newblock Tiny inference-time scaling with latent verifiers.
\newblock \emph{arXiv preprint arXiv:2603.22492}, 2026.

\bibitem[Chen et~al.(2026{\natexlab{a}})Chen, Zhong, Liu, and Du]{skillcat2026}
Kunfeng Chen, Qihuang Zhong, Juhua Liu, and Bo~Du.
\newblock {SkillCAT}: Contrastive assessment and topology-aware skill self-evolution for {LLM} agents.
\newblock \emph{arXiv preprint arXiv:2606.13317}, 2026{\natexlab{a}}.

\bibitem[Chen et~al.(2026{\natexlab{b}})Chen, Xing, Ye, Geng, Lin, Lai, He, Zhai, Gao, and Zhu]{genevolve2026}
Sixiang Chen, Zhaohu Xing, Tian Ye, Xinyu Geng, Yunlong Lin, Jianyu Lai, Xuanhua He, Fuxiang Zhai, Jialin Gao, and Lei Zhu.
\newblock {GenEvolve}: Self-evolving image generation agents via tool-orchestrated visual experience distillation.
\newblock \emph{arXiv preprint arXiv:2605.21605}, 2026{\natexlab{b}}.

\bibitem[Cheng et~al.(2026)Cheng, Liu, Gao, Song, and Tang]{physrag2026}
Kexu Cheng, Zicheng Liu, Mingju Gao, Chunhe Song, and Hao Tang.
\newblock {PhysRAG}: Enhancing physics-awareness in video generation via retrieval-augmented generation.
\newblock In \emph{European Conference on Computer Vision}, 2026.

\bibitem[Dwork et~al.(2015)Dwork, Feldman, Hardt, Pitassi, Reingold, and Roth]{dwork2015reusable}
Cynthia Dwork, Vitaly Feldman, Moritz Hardt, Toniann Pitassi, Omer Reingold, and Aaron Roth.
\newblock The reusable holdout: Preserving validity in adaptive data analysis.
\newblock \emph{Science}, 349\penalty0 (6248):\penalty0 636--638, 2015.

\bibitem[Feng et~al.(2026)Feng, Wang, Xu, Qian, Wang, Hou, Liu, Sun, Liu, and Wang]{newton2026}
Yuxiang Feng, Juncheng Wang, Chao Xu, Yijie Qian, Huihan Wang, Wenlong Hou, Yang Liu, Baigui Sun, Yong Liu, and Shujun Wang.
\newblock {NEWTON}: Agentic planning for physically grounded video generation.
\newblock \emph{arXiv preprint arXiv:2605.18396}, 2026.

\bibitem[Gao et~al.(2026)Gao, Wang, Zhu, Chen, Liu, Bai, Wang, Yuan, Wang, Lu, Cheng, Zhang, Gao, Feng, Liu, Yao, Xu, Zhu, Shen, and Ouyang]{lingbotworld2026}
Zelin Gao, Qiuyu Wang, Jiapeng Zhu, Jingye Chen, Zichen Liu, Qingyan Bai, Jiahao Wang, Yufeng Yuan, Hanlin Wang, Yichong Lu, Ka~Leong Cheng, Haojie Zhang, Jian Gao, Tianrui Feng, Yuzheng Liu, Yao Yao, Yinghao Xu, Xing Zhu, Yujun Shen, and Hao Ouyang.
\newblock Infinite worlds with versatile interactions.
\newblock \emph{arXiv preprint arXiv:2607.07534}, 2026.

\bibitem[Jiang et~al.(2026)Jiang, Meng, Hu, Xie, Xu, and Zhan]{lamo2026}
Bo~Jiang, Depu Meng, Yihan Hu, Yichen Xie, Tianshuo Xu, and Wei Zhan.
\newblock {LaMo}: Self-supervised latent motion priors for physical realism in video generation.
\newblock \emph{arXiv preprint arXiv:2605.23878}, 2026.

\bibitem[Kang et~al.(2026)Kang, Yoon, and Ahn]{ipr2026}
Taegu Kang, Jaesik Yoon, and Sungjin Ahn.
\newblock Inference-time scaling in diffusion models through iterative partial refinement.
\newblock \emph{arXiv preprint arXiv:2605.19317}, 2026.

\bibitem[Li et~al.(2026{\natexlab{a}})Li, Guo, Teng, Shen, and He]{geoalign2026}
Zizun Li, Haoyu Guo, Runzhe Teng, Chunhua Shen, and Tong He.
\newblock {Geo-Align}: Video generation alignment via metric geometry reward.
\newblock \emph{arXiv preprint arXiv:2605.23903}, 2026{\natexlab{a}}.

\bibitem[Li et~al.(2026{\natexlab{b}})Li, Liu, Liu, Zhou, Wu, Chen, Xie, Wu, and Sun]{comfyclaw2026}
Zongxia Li, Dawei Liu, Fuxiao Liu, Yuhang Zhou, Xiyang Wu, Jingxi Chen, Jing Xie, Xiaomin Wu, and Lichao Sun.
\newblock {COMFYCLAW}: Self-evolving skill harnesses for image generation workflows.
\newblock \emph{arXiv preprint arXiv:2607.01709}, 2026{\natexlab{b}}.

\bibitem[Lin et~al.(2026{\natexlab{a}})Lin, Akbari, He, Zhao, Zhang, Akbari, Xu, Lu, Nan, Deng, Yeh, Ostadabbas, Fu, Dy, Zhao, and Wang]{phyground2026}
Juyi Lin, Arash Akbari, Yumei He, Lin Zhao, Haichao Zhang, Arman Akbari, Xingchen Xu, Zoe~Y. Lu, Enfu Nan, Hokin Deng, Edmund Yeh, Sarah Ostadabbas, Yun Fu, Jennifer Dy, Pu~Zhao, and Yanzhi Wang.
\newblock {PhyGround}: Benchmarking physical reasoning in generative world models.
\newblock \emph{arXiv preprint arXiv:2605.10806}, 2026{\natexlab{a}}.

\bibitem[Lin et~al.(2026{\natexlab{b}})Lin, Wu, Wang, Shi, Sang, He, Liu, Wei, Wu, Zhang, Wang, Zhang, Dumoulin, Xie, Zhou, Wang, and Lu]{harnessbenefit2026}
Minhua Lin, Juncheng Wu, Zijun Wang, Zhan Shi, Yisi Sang, Bing He, Zewen Liu, Tianxin Wei, Zongyu Wu, Zhiwei Zhang, Dakuo Wang, Xiang Zhang, Benoit Dumoulin, Cihang Xie, Yuyin Zhou, Suhang Wang, and Hanqing Lu.
\newblock Harness updating is not harness benefit: Disentangling evolution capabilities in self-evolving {LLM} agents.
\newblock \emph{arXiv preprint arXiv:2605.30621}, 2026{\natexlab{b}}.

\bibitem[Ma et~al.(2025)Ma, Tong, Jia, Hu, Su, Zhang, Yang, Li, Jaakkola, Jia, and Xie]{ma2025inference}
Nanye Ma, Shangyuan Tong, Haolin Jia, Hexiang Hu, Yu-Chuan Su, Mingda Zhang, Xuan Yang, Yandong Li, Tommi Jaakkola, Xuhui Jia, and Saining Xie.
\newblock Inference-time scaling for diffusion models beyond scaling denoising steps.
\newblock \emph{arXiv preprint arXiv:2501.09732}, 2025.

\bibitem[{Team Wan} et~al.(2025){Team Wan}, Wang, Ai, Wen, Mao, Xie, Chen, Yu, Zhao, Yang, Zeng, Wang, Zhang, Zhou, Wang, Chen, Zhu, Zhao, Yan, Huang, Feng, Zhang, Li, Wu, Chu, Feng, Zhang, Sun, Fang, Wang, Gui, Weng, Shen, Lin, Wang, Wang, Zhou, Wang, Shen, Yu, Shi, Huang, Xu, Kou, Lv, Li, Liu, Wang, Zhang, Huang, Li, Wu, Liu, Pan, Zheng, Hong, Shi, Feng, Jiang, Han, Wu, and Liu]{wan2025}
{Team Wan}, Ang Wang, Baole Ai, Bin Wen, Chaojie Mao, Chen-Wei Xie, Di~Chen, Feiwu Yu, Haiming Zhao, Jianxiao Yang, Jianyuan Zeng, Jiayu Wang, Jingfeng Zhang, Jingren Zhou, Jinkai Wang, Jixuan Chen, Kai Zhu, Kang Zhao, Keyu Yan, Lianghua Huang, Mengyang Feng, Ningyi Zhang, Pandeng Li, Pingyu Wu, Ruihang Chu, Ruili Feng, Shiwei Zhang, Siyang Sun, Tao Fang, Tianxing Wang, Tianyi Gui, Tingyu Weng, Tong Shen, Wei Lin, Wei Wang, Wei Wang, Wenmeng Zhou, Wente Wang, Wenting Shen, Wenyuan Yu, Xianzhong Shi, Xiaoming Huang, Xin Xu, Yan Kou, Yangyu Lv, Yifei Li, Yijing Liu, Yiming Wang, Yingya Zhang, Yitong Huang, Yong Li, You Wu, Yu~Liu, Yulin Pan, Yun Zheng, Yuntao Hong, Yupeng Shi, Yutong Feng, Zeyinzi Jiang, Zhen Han, Zhi-Fan Wu, and Ziyu Liu.
\newblock Wan: Open and advanced large-scale video generative models.
\newblock \emph{arXiv preprint arXiv:2503.20314}, 2025.

\bibitem[Thomas et~al.(2015)Thomas, Theocharous, and Ghavamzadeh]{thomas2015highconfidence}
Philip~S. Thomas, Georgios Theocharous, and Mohammad Ghavamzadeh.
\newblock High-confidence off-policy evaluation.
\newblock In \emph{Proceedings of the AAAI Conference on Artificial Intelligence}, pp.\  3000--3006, 2015.
\newblock \doi{10.1609/aaai.v29i1.9541}.

\bibitem[Wang et~al.(2026)Wang, Liu, Huang, Huang, Wang, Zhang, Li, and Wu]{sift2026}
Ruoyu Wang, Jialun Liu, Huayang Huang, Haibin Huang, Jiepeng Wang, Chi Zhang, Xuelong Li, and Yu~Wu.
\newblock {SIFT}: Self-imagination fine-tuning for physically plausible motion in video diffusion models.
\newblock In \emph{European Conference on Computer Vision}, 2026.

\bibitem[Wu et~al.(2026)Wu, Cui, Xue, Wang, Luo, Feng, Yang, Wang, Jiang, Zhu, Wang, Nie, Chen, and Wang]{worldreasonbench2026}
Keming Wu, Yijing Cui, Wenhan Xue, Qijie Wang, Xuan Luo, Zhiyuan Feng, Zuhao Yang, Sudong Wang, Sicong Jiang, Haowei Zhu, Zihan Wang, Ping Nie, Wenhu Chen, and Bin Wang.
\newblock {WorldReasonBench}: Human-aligned stress testing of video generators as future world-state predictors.
\newblock \emph{arXiv preprint arXiv:2605.10434}, 2026.

\bibitem[Yang et~al.(2026)Yang, Gong, Huang, Yang, Zhou, Huang, Li, Gao, Dai, Liu, Qiu, Yang, Chen, Yang, and Luo]{skillopt2026}
Yifan Yang, Ziyang Gong, Weiquan Huang, Qihao Yang, Ziwei Zhou, Zisu Huang, Yan Li, Xuemei Gao, Qi~Dai, Bei Liu, Kai Qiu, Yuqing Yang, Dongdong Chen, Xue Yang, and Chong Luo.
\newblock {SkillOpt}: Executive strategy for self-evolving agent skills.
\newblock \emph{arXiv preprint arXiv:2605.23904}, 2026.

\bibitem[Yang et~al.(2025)Yang, Teng, Zheng, Ding, Huang, Xu, Yang, Hong, Zhang, Feng, Yin, Zhang, Wang, Cheng, Xu, Gu, Dong, and Tang]{yang2025cogvideox}
Zhuoyi Yang, Jiayan Teng, Wendi Zheng, Ming Ding, Shiyu Huang, Jiazheng Xu, Yuanming Yang, Wenyi Hong, Xiaohan Zhang, Guanyu Feng, Da~Yin, Yuxuan Zhang, Weihan Wang, Yean Cheng, Bin Xu, Xiaotao Gu, Yuxiao Dong, and Jie Tang.
\newblock {CogVideoX}: Text-to-video diffusion models with an expert transformer.
\newblock In \emph{International Conference on Learning Representations}, 2025.
\newblock URL \url{https://arxiv.org/abs/2408.06072}.

\bibitem[Yin et~al.(2026)Yin, Shi, Guo, and Wang]{vigor2026}
Tengjiao Yin, Jinglei Shi, Heng Guo, and Xi~Wang.
\newblock {VIGOR}: Video geometry-oriented reward for temporal generative alignment.
\newblock \emph{arXiv preprint arXiv:2603.16271}, 2026.

\bibitem[Yuan et~al.(2026)Yuan, Zhang, Friedrich, Beltran-Velez, Hall, Askari-Hemmat, Han, Ballas, Drozdzal, and Romero-Soriano]{wmreward2026}
Jianhao Yuan, Xiaofeng Zhang, Felix Friedrich, Nicolas Beltran-Velez, Melissa Hall, Reyhane Askari-Hemmat, Xiaochuang Han, Nicolas Ballas, Michal Drozdzal, and Adriana Romero-Soriano.
\newblock Inference-time physics alignment of video generative models with latent world models.
\newblock \emph{arXiv preprint arXiv:2601.10553}, 2026.

\bibitem[Zhang et~al.(2026{\natexlab{a}})Zhang, Deng, Sun, Ma, Wang, Du, Pan, Huang, Liang, Huang, Zhang, Xie, Liu, and Zhou]{physisforcing2026}
Peiwen Zhang, Yufan Deng, Shangkun Sun, Juncheng Ma, Duomin Wang, Jonas Du, Zilin Pan, Ye~Huang, Hao Liang, Songyan Huang, Ruihua Zhang, Enze Xie, Ming-Yu Liu, and Daquan Zhou.
\newblock {PhysisForcing}: Physics reinforced world simulator for robotic manipulation.
\newblock \emph{arXiv preprint arXiv:2606.28128}, 2026{\natexlab{a}}.

\bibitem[Zhang et~al.(2026{\natexlab{b}})Zhang, Cui, Wang, Li, Qiu, Zhu, and He]{librarydrift2026}
Xing Zhang, Yanwei Cui, Guanghui Wang, Ziyuan Li, Wei Qiu, Bing Zhu, and Peiyang He.
\newblock Library drift: Diagnosing and fixing a silent failure mode in self-evolving {LLM} skill libraries.
\newblock \emph{arXiv preprint arXiv:2605.19576}, 2026{\natexlab{b}}.

\bibitem[Zhang et~al.(2026{\natexlab{c}})Zhang, Li, Zhang, Gao, Yan, Jiang, Tang, Yin, Wu, Chen, Xu, Shu, Zhang, Xu, Chen, Wang, Liu, Zhou, Zhang, Zhao, and Wu]{qwenimageagent2026}
Zekai Zhang, Jiahao Li, Jie Zhang, Kaiyuan Gao, Kun Yan, Lihan Jiang, Ningyuan Tang, Shengming Yin, Tianhe Wu, Xiaoyue Chen, Xiao Xu, Yan Shu, Yanran Zhang, Yixian Xu, Yuxiang Chen, Zhendong Wang, Zihao Liu, Zikai Zhou, Huishuai Zhang, Dongyan Zhao, and Chenfei Wu.
\newblock {Qwen-Image-Agent}: Bridging the context gap in real-world image generation.
\newblock \emph{arXiv preprint arXiv:2606.26907}, 2026{\natexlab{c}}.

\end{thebibliography}

\appendix

\newpage

\section{Extended Generalization and Prospective Evaluations}
\label{sec:integrated}

\subsection{Prospective PAI Evaluation}
By contrast, the conformal-risk-control stop test asks whether SceneLang should be
replaced on new tasks under a fully frozen gate. It fixes 96 development tasks and a
disjoint 48-task confirmation set
before scoring. The procedure uses strict task-LOO proposals, eight-fold inner
calibration, exact SceneLang fallback, and three fixed CRC thresholds. The merged
development panel contains 288 official score units. At every threshold, the protocol
commits the same two tasks ($2/96$), and both are harmful. The mean gain over the exact
SceneLang anchor is $-0.173$; the threshold-0 simultaneous interval is
$[-0.835,0.000]$, and the conditional harm upper bound is $1.0$. All five CRC
feasibility checks and all three downstream development requirements fail. The procedure therefore retains SceneLang. No videos, VQA outputs, or scores are released for
the 48-task confirmation set.

\subsection{CompletionGuard Successor Evaluation}
A related successor test considers whether a task-disjoint replacement for
CompletionGuard and ContinuityGuard should replace the incumbent over exact SceneLang. It freezes
both guards before
generating three arms on 96 fresh development tasks, while a separate 28-task
confirmation set remains sealed. The 58-dimensional strict task-LOO router makes only two non-fallback proposals
($2/96$), one per candidate arm. Both are non-harmful but improve only QS, with zero RO
gain. At every frozen threshold, the minimum 24 commits fails, the conditional harm upper
bound is $0.935>0.25$, and the simultaneous gain lower bound is exactly zero; the
matched-random lower bound is also non-positive. Thus no CRC threshold is selected and
the secondary RO/QS gates are not entered. The procedure retains exact SceneLang.
We did not generate videos or release VQA for the 28-task confirmation set, and we did
not read its scores; the split6 evolution figure is unchanged.

\subsection{PhyT2V Evaluation}
As a complementary deployment check, we ask whether candidate expansion adds value
beyond retaining the incumbent on a 40-scene PhyT2V evaluation set. Base scores $23.57$,
original-random $29.45$, fixed-round2 $31.71$, the GPT-5.4 incumbent $33.08$, and
selected P-IQ $33.08$. Although five scenes were
eligible for round3, no round3 proposal was promoted. Selected-minus-incumbent is
therefore exactly $0.00$ (CI $[0,0]$), whereas selected-minus-fixed-round2 is $+1.37$
(CI $[-4.04,+7.05]$). The larger $+9.51$ contrast with Base is attributable to the
retained GPT-5.4 proposal, not to \scope candidate expansion. Moreover, only $15/40$
scene IDs are new relative to the preceding evaluation set. We therefore interpret this result as guarded
incumbent retention, not as a positive method effect.
\label{tab:phyt2v-headline}

\subsection{Prospective Integrated-System Evaluation}
Turning from deployment diagnostics to fresh-task tests, we ask whether the full harness
improves on 32 fresh, task-disjoint Physics-IQ tasks and passes its deployment criteria. Full
\scope is $+0.096$ over Base (CI
$[-0.887,+1.077]$) and $-1.019$ versus a simple frozen router (CI
$[-2.159,+0.146]$). Retirement is significantly harmful relative to
commit-plus-fallback: $-0.773$ (CI $[-1.471,-0.203]$). A separate 40-scene
skill-optimization confirmation asks whether the typed update beats matched random. It
does not: the contrast is $-0.32$ (CI $[-5.98,+5.15]$), and the $15\%$ harm rate
exceeds the frozen $10\%$ ceiling. Both deployment criteria fail.

\subsection{Selective Grounded Router}
Against that backdrop, the grounded-router experiment asks whether selective tool use
can avoid the harm of always-on assistance while still outperforming no-tool and random
routing. The prospective
CausalVerse test uses task-disjoint fit, conformal, applicability, target, and
confirmation roles. Every target or confirmation task has the same two candidate videos
(no-tool and profiled-card), three frozen seeds, and official CRONOS DisMo scoring; the
four policies only choose among those shared videos. The fit-only group policy applies
$48$ cards, versus $180$ for always-on and exactly $48$ for a profile/source-stratified
random route.

\begin{table}[H]
\centering
\scriptsize
\caption{\textbf{Selective grounded-router final result.} Absolute DisMo cosine
similarities reuse the same two candidate arms and three seeds for all four policies.
Cards count profiled prompt-card applications across $192$ target and $96$ confirmation
tasks; they are not additional generated videos.}
\label{tab:grounded-router-final}
\setlength{\tabcolsep}{4pt}
\begin{tabular}{@{}lcccc@{}}
\toprule
\textbf{Policy} & \textbf{Cards} & \textbf{Target} & \textbf{Confirmation} & \textbf{Pooled}\\
\midrule
Exact no-tool & $0$ & $0.4071$ & $0.4425$ & $0.4189$\\
Always-on profiled card & $180$ & $0.4021$ & $0.4352$ & $0.4132$\\
Fit-only group policy & $48$ & $0.4096$ & $0.4401$ & $0.4197$\\
Same-budget stratified random & $48$ & $0.4084$ & $0.4424$ & $0.4197$\\
\bottomrule
\end{tabular}
\end{table}

The result is mixed. The selective router avoids the harm of always-on use but does not
establish an absolute generation gain or routing skill beyond random. In DisMo points, the preregistered
pooled contrasts are $+0.083$ for fit-group minus no-tool (CI
$[-0.191,+0.381]$), $+0.660$ for fit-group minus always-on (CI
$[+0.016,+1.315]$), and $+0.002$ for fit-group minus same-budget random (CI
$[-0.281,+0.299]$). The target result is positive: $+0.243$ points overall and
$+1.460$ over the $32$ switched tasks. Independent confirmation, however, reverses
sign to $-0.237$ overall and $-1.420$ over $16$ switched tasks, with negative means
for all three confirmation seeds. Motion/sharpness ratios are $1.001/1.008$, and all
$1{,}728$ scores are present.

We therefore conclude that selective fallback significantly avoids the strong harm of
always-on use, especially for the geometry group, while leaving the no-tool and
same-budget-random contrasts unresolved. Because the evidence calls were acquired before
policy selection and shared across all four policies, the experiment cannot identify
fewer tool calls as the cause of the quality change.

\subsection{PhyGround Applicability Analysis}
Finally, the PhyGround chain asks whether the preregistered selector can clear
applicability before any target score is read. It stops at the pre-score gate on a
32-task target set.
The pipeline completes all $288$ frozen target videos and score-blind features, but the
frozen proxy produces only $4/32$ raw non-incumbent proposals, below the required minimum
of $8$.
It produced $0$ proposals at confidence $\geq-0.25$ (minimum $8$) and $0$ at
confidence $\geq0$ (minimum $4$), while law coverage passed for eight families. The
protocol therefore stops before any target PhyGround score is read and leaves the
21-task confirmation set untouched. This is an applicability failure, not a negative
target-score estimate.

A score-free audit then explains why the gate fails. The main driver is the global
task-max conformal deduction ($58.47$, or $0.731$ after gain normalization) combined
with worst-head selection; motion and sharpness bind on no selected task. The frozen
held-out calibration split likewise produces only $3/64$ raw switches. At that observed
rate, an eight-of-32 gate has probability $8.9\times10^{-5}$. This post-failure replay
diagnoses the frozen design but cannot justify replacing the policy on the consumed pool.

\section{Additional Discussion}
\label{sec:extended-discussion}
The component experiments clarify which parts of the current inference-control space
are useful. Text conditioning and the sampler modification improve their matched
comparisons, and the supervised reward selector outperforms matched random under
scene-grouped OOF evaluation. In contrast, the region-state prior and the matched
web/SAM tool variants do not show reliable gains. These results are important because
they prevent the integrated Physics-IQ improvement from being attributed uniformly to
every component.

The prospective experiments expose a different limitation: reliable deployment
requires both sufficient proposal coverage and a selector that transfers beyond its
development distribution. The selective router avoids the degradation of always-on tool
use while remaining unresolved against no-tool and same-budget random routing, and its
confirmation result reverses direction. PhyGround fails even earlier because the frozen
selector proposes too few non-incumbent actions. These cases motivate future work on
task-disjoint calibration and selectors that can abstain when uncertainty is high,
rather than relaxing thresholds after observing target outcomes.

\section{Reproducibility Details}
\label{sec:extended-claims}
All reported results are linked to the scripts and machine-readable outputs used
to generate the corresponding tables and figures. The release preserves the exact
protocol-local inputs for each experiment and does not aggregate statistics across
incompatible evaluation units.

\section{Evaluation Protocols and Statistical Details}
\label{sec:governance-matrix}
Table~\ref{tab:governance} summarizes the independent unit, evaluation role, and
primary comparison used for each experimental setting. It is provided to make the
statistical units and protocol boundaries explicit.

\begin{table}[t]
\centering
\scriptsize
\caption{\textbf{Result-to-artifact index.} Each row links a reported experimental
result to its smallest authoritative machine-readable source.}
\label{tab:claim-artifacts}
\setlength{\tabcolsep}{2.5pt}
\begin{tabular}{@{}p{0.27\linewidth}p{0.20\linewidth}p{0.47\linewidth}@{}}
\toprule
\textbf{Result} & \textbf{Evaluation role} & \textbf{Authoritative artifact(s)}\\
\midrule
Matched 40-scene Physics-IQ procedure vs exact Base & same-universe system result &
\nolinkurl{physiq_common_base_result.json}; matched publication in the release manifest\\
Retrospective LingBot harness proxy & same-protocol retrospective addition &
matched score evidence and result in the release manifest\\
SceneLang text edit & component; three fixed seeds &
\nolinkurl{scenelang_skill_table.json}\\
Fixed sampler edit & component; four fixed seeds &
\nolinkurl{sampler_axis_joint_fourseed_result.json}\\
Supervised reward selector & scene-grouped OOF &
\nolinkurl{reward_gain_ci.json}\\
Integrated and new-task prospective failures & prospective negative / stop &
\nolinkurl{pai_scope_lifecycle_staircase_result.json};
\nolinkurl{scope_v6_fit_group_final_result.json};
\nolinkurl{pai_wide_crc_human096_development_completion_r7.json};
\nolinkurl{pai_wide_completion_guard_development_completion.json}\\
\bottomrule
\end{tabular}
\end{table}

\begin{table}[t]
\centering
\scriptsize
\caption{\textbf{Protocol and statistical status matrix.} Views are nested within
scenes for Physics-IQ; secondary, same-split, and post-hoc rows never receive promotion
credit.}
\label{tab:governance}
\setlength{\tabcolsep}{2.5pt}
\begin{tabular}{@{}p{0.19\linewidth}p{0.18\linewidth}p{0.22\linewidth}p{0.33\linewidth}@{}}
\toprule
\textbf{Setting} & \textbf{Independent unit} & \textbf{Evaluation role} & \textbf{Primary comparison}\\
\midrule
GPT-5.4 SceneLang & scene; 3 fixed seeds averaged & component evidence & SceneLang vs Base; CI conditional on these seeds\\
Learned reward & 64 scenes; 179 nested views & component evidence & scene-grouped OOF reward vs matched random\\
Physics-IQ tool axis & scene; 3 nested views & preregistered negative ablation & web-only and web+SAM vs no-tool; incremental SAM vs web-only\\
PAI split6 & task & same-split deployment diagnostic & all matched controls shown; post-hoc \scope separated\\
OpenS2V & category-stratified item & prospective deployment case & \scope vs LingBot; unresolved\\
Fresh PhyT2V evaluation & scene & deployment diagnostic & selected vs incumbent/fixed proposal; no expansion credit\\
Fresh 32-task Physics-IQ transfer & task & prospective integrated test & full harness vs Base/router; frozen guard gate\\
Typed-update confirmation & scene & prospective integrated test & typed update vs matched random plus harm ceiling\\
Grounded selective router & task; role/profile/family bootstrap & prospective target+confirmation; STOP & fit-group vs no-tool, always-on, and same-budget random\\
\bottomrule
\end{tabular}
\end{table}

\section{Full PAI Metric Breakdown}
\label{sec:pai-8d}
Table~\ref{tab:pai-8d} exposes the quality components aggregated by QS in the main
PAI panel. Frozen Base is shown in absolute score units; every other row is a signed
difference from that same Base. These component-wise values are descriptive diagnostics,
not separately tested deployment criteria.

\begin{table}[t]
\centering
\scriptsize
\caption{\textbf{Complete official-72B PAI quality profiles.} SC/BC are subject/background
consistency; MS is motion smoothness; AQ/IQ are aesthetic/imaging quality; OC is overall
text--video consistency; IS/IB are I2V subject/background fidelity. Frozen Base is
absolute and all other rows are $\Delta$ vs Base. All 12 main-panel rows are included.}
\label{tab:pai-8d}
\setlength{\tabcolsep}{1.5pt}
\resizebox{\linewidth}{!}{%
\begin{tabular}{@{}lcccccccc@{}}
\toprule
\textbf{Method} & \textbf{SC} & \textbf{BC} & \textbf{MS} & \textbf{AQ} &
\textbf{IQ} & \textbf{OC} & \textbf{IS} & \textbf{IB}\\
\midrule
\csname @@input\endcsname pai_full_metrics_unified72_eight_dimension_rows.tex
\bottomrule
\end{tabular}}
\end{table}

\section{Additional Qualitative Analysis}
\label{sec:posthoc-audits}
The following visualization illustrates a disclosed post-hoc paper-submersion
intervention. It is included only as qualitative analysis and is not a native generation
result or benchmark comparison.

\begin{figure}[t]
\centering
\includegraphics[width=\linewidth]{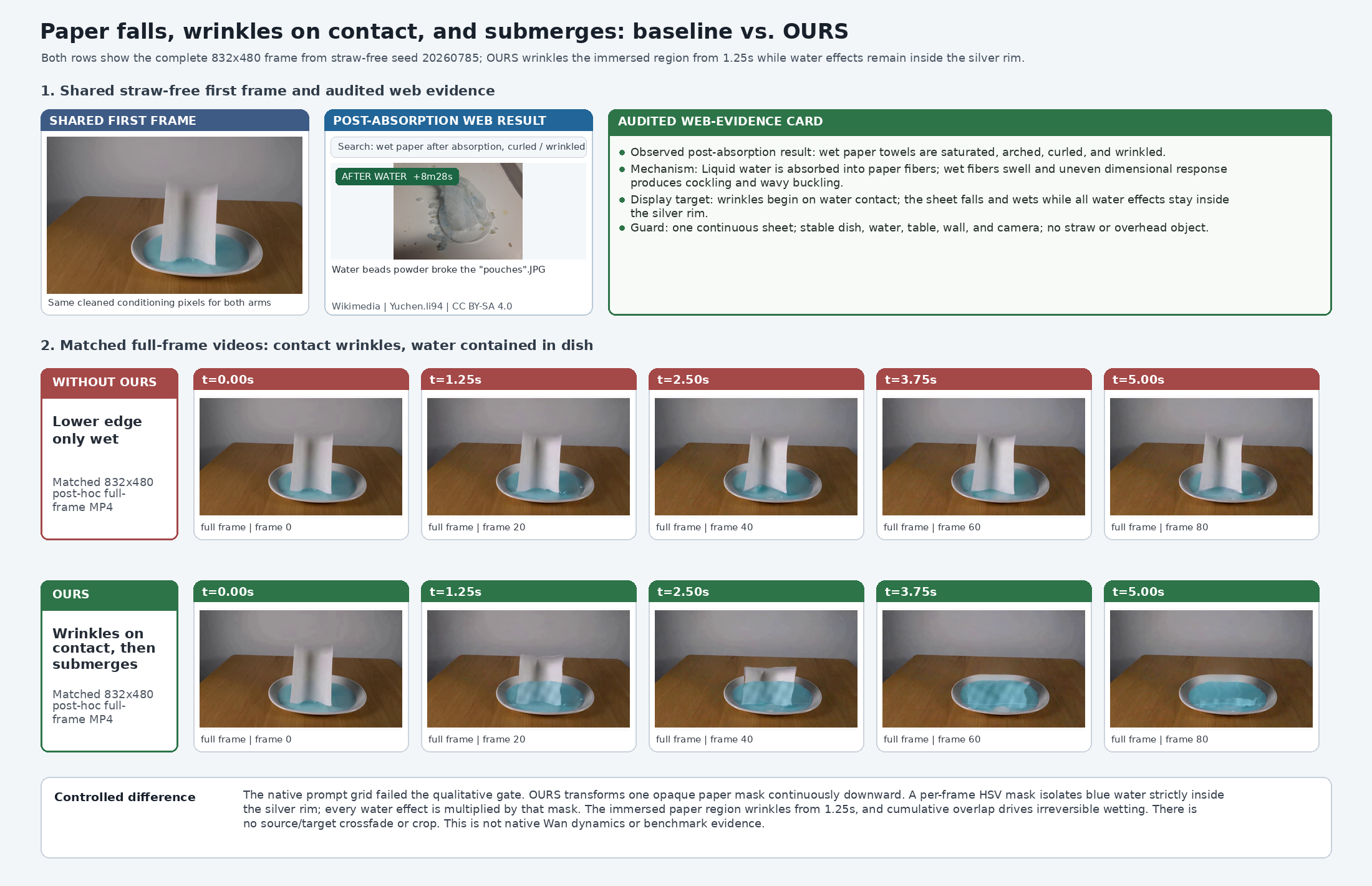}
\caption{\textbf{Web-grounded complete paper submersion intervention.}
Both rows use the same straw-free seed-$20260785$ Wan2.2 baseline. The unexplained
overhead straw is inpainted once in their shared conditioning frame.
\textbf{WITHOUT OURS} shows the complete baseline frame, in which only the lower paper
edge contacts the water. \textbf{OURS} applies a disclosed post-hoc full-frame
intervention: the complete paper mask is compressed onto the dish bottom; every final
paper pixel lies within the audited water ellipse; and water tint, caustics, refraction,
and a ripple band are rendered after the paper. Both published A/B MP4s and the displayed
frames are complete $832\!\times\!480$ images, with dish cropping disabled in the
manifest. The post-absorption Wikimedia image is display-only and never enters either
video. Retrieved pixels likewise never enter the video. The native four-seed prompt grid
failed the qualitative gate and is not shown. This visualization is a qualitative
post-hoc state intervention, not native Wan dynamics, an unbiased generation result, or
benchmark evidence.}
\label{fig:sam-web-case}
\end{figure}

\end{document}